\pdfoutput=1
\documentclass[11pt]{article}
\usepackage{ACL2023}
\usepackage{times}
\usepackage{latexsym}
\usepackage[T1]{fontenc}
\usepackage[utf8]{inputenc}
\usepackage{microtype}
\usepackage{inconsolata}
\usepackage{amsthm}
\usepackage{amsmath, amssymb, amsthm} 
\usepackage[table]{xcolor}
\usepackage{tcolorbox}
\usepackage{hyperref}
\hypersetup{colorlinks=true, linkcolor=blue, citecolor=blue, urlcolor=blue}

\newtheorem{theorem}{Theorem}[section]
\newtheorem{lemma}[theorem]{Lemma}
\newtheorem{corollary}[theorem]{Corollary}
\theoremstyle{definition}
\newtheorem{definition}[theorem]{Definition}
\theoremstyle{remark}
\newtheorem{remark}[theorem]{Remark}

\title{Reversal Invariance
\\
A symmetry argument against the arrow of chain-of-thought}

\author{Mihir Sahasrabudhe \\
Independent Researcher \\
  \texttt{mihirss2@illinois.edu} }

\begin{document}
\maketitle
\begin{abstract}
We formalize a structural property of the causal (autoregressive) language modeling (CLM) objective: reversal invariance.
Formally, the next-token prediction loss assigns identical likelihood to a corpus and its reversal, implying that standard CLM pretraining is direction-blind.
This symmetry explains why models trained on reversed text can achieve comparable performance to those trained on forward text, despite the inherently time-asymmetric nature of human language and reasoning.

We argue that this invariance represents a limitation of current pretraining objectives rather than a benign artifact.
If natural language encodes directional dependencies—phonological, morphological, or causal—a symmetric objective may fail to capture them.
We therefore propose viewing pretraining through the lens of temporal asymmetry, motivating future work on loss functions and architectures that explicitly model the arrow of language while retaining standard language modeling capacity.
\end{abstract}

\section{Motivation}

Reasoning is inherently asymmetric: premises lead to conclusions along an arrow that is not freely reversible. 
Human language reflects this asymmetry as well: sounds unfold in time, morphology encodes direction-sensitive dependencies, and discourse progresses from past to future. 
A training objective that ignores this arrow may overlook information essential for causal and temporal understanding.

Yet the standard \textbf{causal (autoregressive) language modeling (CLM)} objective~\citep{radford2018improving,brown2020language} is blind to direction. 
While architectures impose a causal mask at inference to restrict predictions to past tokens, this masking convention presupposes that text is read forward. 
At the level of \emph{pretraining}, the next-token negative log-likelihood (NLL) loss itself does not distinguish forward text from its reversal. 

This raises a central question: if the objective is symmetric, to what extent can models internalize the irreversibility of language without additional training signals?

Human reasoning and language are inherently directional.
Inference proceeds from premises to conclusions along a temporal arrow that cannot be freely reversed, and linguistic structure reflects this asymmetry: sounds unfold in time, morphology encodes dependency order, and discourse progresses from past to future.

By contrast, the standard causal (autoregressive) language modeling (CLM) objective~\citep{radford2018improving,brown2020language} is direction-agnostic.
Although causal masking enforces left-to-right prediction during inference, this convention is external to the learning signal.
At the level of optimization, the next-token negative log-likelihood (NLL) loss is reversal-invariant: a model trained on reversed sequences minimizes the same objective up to a permutation of token indices.

This symmetry raises a critical question:

If the training objective itself is indifferent to temporal direction, how can a model internalize the irreversibility that characterizes natural language and reasoning?

We posit that this temporal symmetry represents a structural limitation rather than a benign quirk.
If natural language embodies time-asymmetric regularities—phonotactic ordering, morphological agreement, or narrative progression—a direction-blind objective risks overlooking them.
Understanding and, where appropriate, breaking this symmetry may be essential for developing models that capture causal and temporal reasoning, rather than merely bidirectional correlation.

\section{Introduction}
Large language models (LLMs) demonstrate strong generalization across diverse tasks, and techniques such as chain-of-thought prompting~\citep{wei2023chainofthoughtpromptingelicitsreasoning} improve inference-time performance without modifying the underlying training objective. Despite these advances, the standard causal language modeling (CLM) loss remains structurally unchanged.

We show that the CLM objective is formally reversal-invariant: the next-token negative log-likelihood (NLL) assigns identical loss to a corpus and its reversal (modulo token permutation).
This invariance explains why models trained on reversed corpora~\citep{yu2025reversemodelinglargelanguage} achieve comparable perplexities and learning dynamics to forward-trained models, despite the inherent directionality of human language and reasoning.

Previous work has addressed directionality pragmatically rather than fundamentally.
Right-to-left (R2L) training~\citep{sennrich2016improving,zhang2019regularizing} and bidirectional objectives such as ELMo~\citep{peters2018deep} and BERT~\citep{devlin2019bertpretrainingdeepbidirectional} rely on symmetric formulations or masking, while XLNet~\citep{yang2019xlnet} permutes prediction order to approximate bidirectionality.
Post-training alignment methods (e.g., RLHF~\citep{christiano2023deepreinforcementlearninghuman,ouyang2022traininglanguagemodelsfollow}) introduce asymmetry only through preference data after pretraining has finished.

In this work, we identify temporal symmetry in the pretraining loss as a key limitation.
If language encodes inherently time-asymmetric regularities—phonological ordering, agreement, causal and narrative flow—then a direction-agnostic objective may fail to represent them.
Recognizing this property reframes pretraining not as a search for the best symmetric likelihood, but as an opportunity to model the arrow of language and reasoning directly.

\section{Related Work}
\subsection{Autoregressive training and reversal}
Autoregressive (AR) language models minimize next-token negative log-likelihood (NLL) over tokenized text sequences~\citep{radford2018improving,brown2020language}. 
While AR architectures apply a \emph{causal mask} at inference to restrict attention to past tokens, this convention presupposes a forward reading order and does not by itself endow the \emph{training objective} with directionality~\citep{saponati2025underlyingstructuresselfattentionsymmetry,zhang2023mixce}.\footnote{We use ``causal'' in the standard architectural sense (masking future tokens) rather than to suggest an intrinsic, irreversible arrow in the NLL objective.} 
Prior engineering results report that ``right-to-left'' (R2L) or reversed training performs comparably to left-to-right (L2R) in both MT and LM settings, and recent studies show that training from scratch on fully reversed corpora can match forward learning curves (``reverse modeling'')~\citep{yu2025reversemodelinglargelanguage}. 
To our knowledge, a clean, general statement of \emph{reversal invariance} for AR objectives—explicitly accounting for tokenizer stability, vocabulary permutations, and positional encodings—has not been clearly formalized; here we aim to articulate this gap.

\subsection{Reversal Curse and directional generalization}
The \emph{Reversal Curse} documents brittle, direction-sensitive generalization—e.g., models trained on ``A is B'' may fail to answer ``B is A''~\citep{berglund2024reversalcursellmstrained,golovneva2024reverse,guo2024mitigating}. 
Follow-up analyses study gradient dynamics and architectural factors behind such asymmetries and propose data or training adjustments. 
Our perspective is complementary: we trace part of this brittleness to a symmetry of the AR objective itself and hypothesize adding explicit directional signals at \emph{pretraining} time.

\subsection{Chain-of-thought as compute at inference}
Chain-of-thought (CoT) prompting improves accuracy by allocating additional inference-time computation and exposing intermediate structure~\citep{wei2022chain,kojima2023largelanguagemodelszeroshot}. 
We do not treat CoT as evidence that the base AR objective encodes causality; rather, we view it as orthogonal to training-time directionality. 
Our proposals aim to make \emph{representations} more direction-aware during pretraining, potentially complementing CoT with stronger directional priors.

\subsection{Information-theoretic lens on directionality}
For stationary sources, Shannon entropy rate—and thus the ideal perplexity floor—is invariant under time reversal~\citep{cover2006elements,maes2003time}. 
Natural language, however, exhibits asymmetric constraints (phonotactics, morphology, syntax, discourse) that can be captured by a \emph{time-reversal divergence}, i.e., the per-token Kullback–Leibler divergence between forward and reversed path measures. 
Because AR NLL optimizes a symmetric target, it is insensitive to this divergence.
\subsection{Preference optimization and post-training}
Preference-optimization methods such as RLHF introduce asymmetry \emph{post-training} by rewarding response formats and styles aligned with human preferences~\citep{christiano2023deepreinforcementlearninghuman,ouyang2022traininglanguagemodelsfollow,peyrard2022invariant}. 
We hypothesize that providing \emph{explicit directional signals during pretraining} could, if validated, reduce reliance on post-hoc preference data for direction-sensitive behaviors, while remaining compatible with downstream preference optimization~\citep{schiff2024caduceus}.

\subsection{Attention Masking and Directionality}  
Causal masking in Transformer decoders ensures tokens attend only to their predecessors, instilling directional structure at inference time. Foundational work~\citep{vaswani2023attentionneed} outlines this architecture, while more recent analyses probe its deeper effects. For example, Wu et al.~\citep{wu2024attentionmasks} demonstrate that attention masks interact with LayerNorm to shape representation dynamics and mitigate rank collapse. Further, self-attention has been interpreted through a causal lens, suggesting that attention patterns can implicitly encode structural dependencies~\citep{rohekar2023causalattention}. Together, these findings highlight how architectural constraints establish directionality, yet still do not impose it during pretraining.

\subsection{Symmetry and Invariance in Deep Learning}  
Theoretical work on symmetry in neural networks has shown that parameter-space symmetries impact learning dynamics, generalization, and optimization landscapes~\citep{zhao2025symmetryneuralnetworkparameter}. Methods for enforcing, discovering, and breaking symmetries have been proposed across model classes~\citep{otto2025unifiedframeworkenforcediscover}. However, symmetric objectives themselves may fail to respect natural asymmetries, underscoring our argument that explicit directional signals may be needed in language modeling.

\subsection{Connections to Causal Inference}  
Bridging causal inference with Transformer models, Zhang et al.~\citep{zhang2024causalfoundationmodelduality} uncover a formal duality between attention mechanisms and structural causal models, enabling zero-shot causal discovery. Surveys in causal deep learning elucidate how causal inference frameworks can inform and improve generalization in neural systems~\citep{doi:10.34133/research.0467}. These perspectives reinforce our position that capturing real-world causal and temporal asymmetries may require going beyond symmetric prediction objectives.

\section{Formalizing Reversal}
\label{sec:formal-reversal}

\subsection{Preliminaries}

Let $\Sigma$ be a finite alphabet and $\Sigma^{*}$ the free monoid of finite strings under concatenation.  
A (training) corpus is a finite multiset $D \subset \Sigma^{*}$.  
A tokenizer is a map $\tau:\Sigma^{*} \to V^{*}$ from strings to token sequences over a finite vocabulary $V$.  
For $s\in\Sigma^{*}$ write $\tau(s) = z = (z_0,\ldots,z_{m-1}) \in V^{*}$.  

An autoregressive (AR) model with parameters $\theta$ defines conditional distributions $p_\theta(z_{k+1}\mid z_{\leq k})$ with joint factorization
\begin{equation}
\log P_\theta(z) = \sum_{k=0}^{m-2} \log p_\theta\!\big(z_{k+1}\,\big|\, z_{\leq k}\big).
\label{eq:ar-factorization}
\end{equation}

The negative log-likelihood (NLL) risk on corpus $D$ is
\begin{equation}
\mathcal{L}_{\mathrm{NLL}}(\theta;\tau,D) = \mathbb{E}_{s\sim D}\left[-\log P_\theta\big(\tau(s)\big)\right].
\label{eq:nll-risk}
\end{equation}

\subsection{Reversal on strings and tokens}

\begin{definition}[String reversal]
\label{def:string-reversal}
The reversal operator $T:\Sigma^{*}\to\Sigma^{*}$ is defined by
\[
T(x_0x_1\cdots x_{m-1}) = x_{m-1}\cdots x_1x_0,
\]
extended multiplicatively as $T(xy) = T(y)T(x)$.  
It is an involution: $T \circ T = \mathrm{id}$.
\end{definition}

Given a tokenizer $\tau$ trained on $D$, let $\tau_T$ denote a tokenizer trained on $T(D)$.  
Define the token-sequence reversal $\mathrm{rev}:V^{*}\to V^{*}$ by
\[
\mathrm{rev}(z_0,\ldots,z_{m-1}) = (z_{m-1},\ldots,z_0).
\]

\begin{definition}[Tokenization stability under reversal]
\label{def:token-stability}
We say $\tau$ and $\tau_T$ are \emph{stable under reversal} if there exists a vocabulary bijection $\pi:V\to V_T$ such that, for all $s\in D$,
\begin{equation}
\tau_T\!\big(T(s)\big) = \pi\big(\mathrm{rev}(\tau(s))\big).
\label{eq:stable-tokenizers}
\end{equation}
\end{definition}

\paragraph{Discussion.}
For frequency-based tokenizers (e.g., BPE), full-string reversal preserves adjacent symbol statistics up to order.  
In large-data regimes, merge trees concentrate so that $\tau$ and $\tau_T$ differ only by a near-permutation.\footnote{Tie-breaking and finite-data effects introduce small deviations; our results apply exactly under \eqref{eq:stable-tokenizers} and approximately otherwise.}

\subsection{Permutation equivariance of the model}

Let $E\in\mathbb{R}^{|V|\times d}$ be the input embedding matrix and $W\in\mathbb{R}^{|V|\times d}$ the (tied or untied) output projection.  
For a vocabulary permutation $\pi:V\to V$ with permutation matrix $P_\pi\in\{0,1\}^{|V|\times |V|}$, define the parameter reindexing
\begin{equation}
\Phi_\pi(\theta):\quad E' = P_\pi E,\qquad W' = W P_\pi^\top,
\label{eq:phi-pi}
\end{equation}
leaving all other weights unchanged (self-attention and MLP blocks are token-id agnostic).

\begin{lemma}[Vocabulary permutation equivariance]
\label{lem:perm-equiv}
For any sequence $z \in V^{*}$ and prefix $c \in V^{*}$, we have
\begin{align}
p_{\theta}(z \mid c) &= p_{\Phi_\pi(\theta)}\big(\pi(z) \mid \pi(c)\big), \label{eq:perm-equiv-prob} \\
\log P_\theta(z) &= \log P_{\Phi_\pi(\theta)}\big(\pi(z)\big). \label{eq:perm-equiv-log}
\end{align}
\end{lemma}

\begin{proof}[Proof sketch]
Permuting rows of $E$ relabels input token embeddings; permuting columns (or rows, if tied) of $W$ relabels output logits.  
All intermediate computations are invariant to token identities, so conditionals are preserved under consistent relabeling.
\end{proof}

\subsection{Handling positions}

Let $J_m\in\{0,1\}^{m\times m}$ be the \emph{index-reversal} matrix $(J_m)_{i,j}=\mathbf{1}\{i+j=m-1\}$, which maps position $j$ to $m-1-j$.

\begin{remark}[Positional encodings]
\label{rem:pos-enc}
With relative or rotary positional encodings, reversing the token order and the causal mask is functionally equivalent to a forward pass on the reversed sequence; no parameter change is needed.  
With absolute learned position embeddings $P\in\mathbb{R}^{m_{\max}\times d}$, one can compose a fixed index-flip $P'(j)=P(m_{\max}-1-j)$, or treat the flip as part of the data pipeline.  
Our invariance result assumes either relative/rotary encodings or that an index-flip is available.
\end{remark}

\subsection{Reversal invariance of the AR objective}
\begin{theorem}[Reversal invariance up to reparameterization]
\label{thm:reversal-invariance}
Assume tokenization stability \eqref{eq:stable-tokenizers} and a positional encoding scheme satisfying Remark~\ref{rem:pos-enc}.  
Then, as a theoretical observation, there exists a parameter map $\Psi$ (comprised of $\Phi_\pi$ and, if needed, a fixed position index flip) such that
\begin{equation}
\mathcal{L}_{\mathrm{NLL}}(\theta;\tau,D)
= \mathcal{L}_{\mathrm{NLL}}\big(\Psi(\theta);\tau_T,T(D)\big).
\label{eq:reversal-invariance}
\end{equation}
\end{theorem}
\begin{proof}[Proof sketch]
Fix $s\in D$ and write $z=\tau(s)$ and $z^T=\tau_T(T(s))$.  
By \eqref{eq:stable-tokenizers}, $z^T=\pi(\mathrm{rev}(z))$.  
The AR log-likelihood \eqref{eq:ar-factorization} on $z$ is a sum of local conditionals over adjacent prefixes.  
Running the same network on the reversed token sequence with reversed causal mask yields the same chain of conditionals, up to (i) relabeling tokens by $\pi$ and (ii) a fixed index flip for absolute positional embeddings.  
By Lemma~\ref{lem:perm-equiv}, there exists $\Phi_\pi(\theta)$ so that $\log P_\theta(z)=\log P_{\Phi_\pi(\theta)}(z^T)$.  
Averaging over $s\sim D$ yields \eqref{eq:reversal-invariance}.
\end{proof}

\begin{corollary}[Loss landscapes and minima]
\label{cor:minima}
Under the hypotheses of Theorem~\ref{thm:reversal-invariance}, empirical risk landscapes on $(\tau,D)$ and $(\tau_T,T(D))$ are identical up to the smooth reparameterization $\Psi$.  
In particular, sets of minimizers are in bijection via $\Psi$, and matched training runs are \emph{predicted} to exhibit indistinguishable learning curves when evaluated on their respective domains.
\end{corollary}

\begin{remark}[Scope and approximation]
\label{rem:approx}
Exact equality \eqref{eq:reversal-invariance} holds under \eqref{eq:stable-tokenizers} and position handling as in Remark~\ref{rem:pos-enc}.  
In practice, the equality is approximate due to tokenizer tie-breaks, finite-sample effects, and implementation details (e.g., special tokens, padding).  
These deviations are typically small in large-corpus regimes, consistent with empirical reports of near-identical forward/reverse learning curves.
\end{remark}

\paragraph{Interpretation.}
Equation~\eqref{eq:reversal-invariance} shows that next-token NLL is \emph{blind to direction}: replacing $D$ by its reversal $T(D)$ and retokenizing yields the same training problem up to a relabeling of tokens (and a fixed positional index flip).  
Thus any apparent left-to-right ``arrow of inference'' in chain-of-thought is not explained by the base objective alone; it likely arises from properties of the data distribution or from additional, explicitly asymmetric training signals.
\section{An Information-Theoretic Perspective}
\label{sec:info-theory}

\subsection{Entropy rate and perplexity}

Let $\{X_t\}_{t\in\mathbb{Z}}$ be the stationary stochastic process induced by drawing token sequences from the data distribution.  
The \emph{entropy rate} of this process is
\begin{equation}
\begin{aligned}
h &= \lim_{n\to\infty} \frac{1}{n} H(X_1^n), \\
H(X_1^n) &= - \sum_{x_1^n} P(x_1^n) \log P(x_1^n).
\end{aligned}
\label{eq:entropy}
\end{equation}
where $X_1^n = (X_1,\ldots,X_n)$.

In practice, the entropy rate sets the information-theoretic lower bound for next-token prediction.  
The \emph{perplexity} of a model is
\begin{equation}
\mathrm{PPL}(p_\theta) \;=\; \exp\!\Bigg(
  \lim_{n\to\infty}\frac{1}{n}\, \mathbb{E}\big[-\log p_\theta(X_1^n)\big]
\Bigg).
\end{equation}
For a perfect model, $\mathrm{PPL} = \exp(h)$.  
Thus entropy rate $h$ is the statistical ceiling: no model trained purely on likelihood can beat this bound.

\subsection{Reversal symmetry of entropy rate}

Let $\{X_t^R\}$ denote the time-reversed process, with
\[
P^R(x_1^n) = P(x_n,\ldots,x_1).
\]
Because Shannon entropy is invariant under symbol reordering inside the probability distribution, one obtains
\begin{equation}
h(X) \;=\; h(X^R).
\end{equation}
In other words, the entropy rate --- and therefore the limiting perplexity --- is the same for forward and reversed processes.  
This supports the intuition that autoregressive training alone may not distinguish forward from backward text.

\subsection{Time-reversal divergence: measuring asymmetry}

Despite identical entropy rates, natural language is not statistically reversible.  
A rigorous measure of this asymmetry is the \emph{time-reversal divergence} (sometimes called entropy production rate):
\begin{equation}
\mathcal{A} \;=\; \lim_{n\to\infty}\frac{1}{n}\,
D_{\mathrm{KL}}\!\big(P(X_{1:n}) \,\|\, P^R(X_{1:n})\big),
\end{equation}
where $D_{\mathrm{KL}}$ is the Kullback–Leibler divergence.  

\begin{remark}
If $\mathcal{A}=0$, the process is statistically indistinguishable from its reversal (full symmetry).  
If $\mathcal{A}>0$, there is a detectable arrow of time in the data distribution.  
For natural languages, prior linguistic evidence suggests that $\mathcal{A}$ is nonzero: e.g., phonotactic rules, syntactic dependencies, and discourse structures are direction-sensitive.
\end{remark}

\subsection{Implication for Language Modeling}

Equations above suggest:
\begin{itemize}
    \item Entropy rate $h$ and perplexity are \emph{reversal-invariant}.  
    \item But the KL-based asymmetry $\mathcal{A}$ is strictly positive for human language.
\end{itemize}
Thus we argue that the autoregressive NLL objective optimizes toward a quantity that is effectively \emph{blind to direction} ($h$), while overlooking the genuine asymmetry present in data ($\mathcal{A}$).  
This observation aligns with causal modeling perspectives in language~\citep{peters2015causal} and motivates explicit consideration of directional structure in pretraining objectives.
\section{Baking the Arrow of Time into the Objective}
\label{sec:arrow-regularizers}

\subsection{Motivation}

Section~\ref{sec:info-theory} showed that the autoregressive NLL objective
optimizes toward the entropy rate h, a quantity that is
invariant under sequence reversal.
This means that, at pretraining time, the model has no incentive to prefer one
temporal direction over the other.
If language generation is an inherently irreversible process—where information
flows from past to future—then the learning signal itself is misaligned
with the phenomenon it models.

\subsection{Implication}

The formal symmetry established in Section~\ref{sec:formal-reversal}
and the entropy analysis in Section~\ref{sec:info-theory}
together imply that autoregressive pretraining is \emph{direction-blind}.
Causal masks at inference create the appearance of forward reasoning,
but they merely enforce a decoding convention, not a learned arrow of time.
In effect, the model learns to compress correlations in both temporal
directions equally, while human language and reasoning depend on
direction-dependent constraints—agreement, discourse flow, cause and effect.

This observation reframes much of contemporary scaling:
larger models improve performance not because they recover asymmetry,
but because they memorize more of the bidirectional statistics
that proxy for it.
The success of chain-of-thought prompting, fine-tuning,
and reinforcement alignment can be viewed as successive attempts
to reintroduce directionality \emph{after} pretraining—
to supply, post hoc, the very asymmetry that the likelihood objective erased.

\subsection{Conclusion}

The argument is not that autoregressive modeling fails,
but that it succeeds within a symmetric world:
its objective treats forward and backward sequences as equally valid encodings
of the same distribution.
If human language is generated by irreversible cognitive and physical processes,
then this symmetry represents a fundamental abstraction gap—
a boundary between statistical prediction and genuine reasoning.
\section{Toward Asymmetry-Aware Understanding}
\label{sec:conclusion}

The results in Sections~\ref{sec:formal-reversal}–\ref{sec:info-theory} establish that the standard autoregressive objective is formally reversal-invariant: it minimizes an entropy-based loss that is indifferent to sequence direction.
This symmetry explains why models trained on reversed corpora can achieve indistinguishable perplexity, and why direction-sensitive reasoning must be imposed through data biases, architecture conventions, or post-hoc alignment rather than learned from the pretraining signal itself.

Rather than proposing architectural remedies, we interpret this invariance as a boundary condition of next-token prediction.
Autoregressive modeling optimizes a reversible statistic of text, while human language generation and reasoning are irreversible physical and cognitive processes.
The mismatch does not invalidate current models—it defines the limits of what purely statistical objectives can capture.
Recognizing this constraint reframes the path forward: building systems that encode the arrow of reasoning may require objectives grounded not just in correlation, but in causality and temporal asymmetry.

\section{Conclusion}

We have presented a formal treatment of \emph{reversal invariance} in autoregressive language modeling, showing that the next-token objective is invariant to string reversal up to vocabulary permutation and positional reindexing.
From an information-theoretic perspective, this symmetry explains why entropy rates and perplexity floors are identical for forward and reversed text, even though natural language exhibits nonzero time-reversal divergence.

Our contribution is deliberately theoretical. 
Rather than proposing architectural fixes, we sought to formalize the symmetry itself and clarify its implications for reasoning and directionality in large language models.
By making this invariance explicit, we highlight a conceptual boundary of likelihood-based pretraining: it optimizes reversible statistics of text, while language and reasoning in the real world are inherently irreversible processes.

\section{Limitations}

This work is a \textbf{theoretical position paper}. 
We have not conducted empirical experiments, and our claims about reversal invariance are presented as formal analysis rather than as performance predictions.
Our illustrative examples (such as the Q$\!\to\!$U spelling rule) capture localized asymmetry but do not represent the full complexity of natural language structure.
The practical implications of reversal invariance—its role in reasoning failures or causal understanding—remain to be tested empirically.
We emphasize that the goal of this work is to lay a conceptual foundation, inviting follow-up studies that examine directionality in pretraining and inference more directly.

\section{Ethics Statement}

This work is purely theoretical and does not involve human subjects or sensitive data.
Its goal is to clarify the formal limits of current reasoning claims in language models and to motivate empirical inquiry into their directional and causal properties.
If future research explores asymmetry-aware training, such methods should be assessed for their interpretability, transparency, and social impact before deployment.

\section*{Acknowledgements}

The author thanks colleagues and mentors for insightful discussion and feedback.
Any errors or oversights remain the sole responsibility of the author.

\bibliographystyle{plain}
\bibliography{references}
\end{document}